\documentclass{article}

\usepackage{microtype}
\usepackage{graphicx}
\usepackage{subcaption}
\usepackage{booktabs} 

\usepackage{hyperref}

\usepackage[preprint]{icml2026}

\usepackage{amsmath}
\usepackage{amssymb}
\usepackage{mathtools}
\usepackage{amsthm}
\usepackage{natbib} 
\usepackage{mathtools, amsfonts}
\usepackage{multicol}
\usepackage{caption}
\usepackage{subcaption}
\usepackage{booktabs}  
\usepackage{makecell}  
\usepackage{graphicx}  

\usepackage[capitalize,noabbrev]{cleveref}

\theoremstyle{plain}
\newtheorem{theorem}{Theorem}[section]

\theoremstyle{definition}

\theoremstyle{remark}
\newtheorem{remark}[theorem]{Remark}
\newtheorem{prop}{Proposition}[section]

\newcommand{\vecX}{X}
\newcommand{\vecY}{Y}
\newcommand{\vecZ}{Z}

\newcommand{\best}[1]{\mathbf{#1}}      
\newcommand{\second}[1]{\underline{#1}} 
\newcommand{\metric}[3]{\ensuremath{#3{#1} \pm #2}}

\newcommand{\metricplain}[2]{\metric{#1}{#2}{}}
\newcommand{\metricbest}[2]{\metric{#1}{#2}{\best}}
\newcommand{\metricsecond}[2]{\metric{#1}{#2}{\second}}

\usepackage[textsize=tiny]{todonotes}

\begin{document}

\twocolumn[
    \icmltitle{ASK-NN: An Asymmetric Nearest-Neighbor Test that detects Distribution Drifts in Natural Language}



  \icmlsetsymbol{equal}{*}

  \begin{icmlauthorlist}
    \icmlauthor{Sergey Zakharov}{equal,yyy}
    \icmlauthor{Rodion Oblovatny}{equal,yyy,sch}
    \icmlauthor{Alexey Zaytsev}{comp,sch}
  \end{icmlauthorlist}

  \icmlaffiliation{yyy}{Saint-Peterburg university, Saint-Peterburg, Russia}
  \icmlaffiliation{comp}{Applied AI Institute, Moscow, Russia}
  \icmlaffiliation{sch}{Risk AI Research Lab, Moscow, Russia}

  \icmlcorrespondingauthor{Firstname1 Lastname1}{first1.last1@xxx.edu}
  \icmlcorrespondingauthor{Firstname2 Lastname2}{first2.last2@www.uk}

  \icmlkeywords{Machine Learning, ICML}

  \vskip 0.3in
]



\printAffiliationsAndNotice{}  

\begin{abstract}
Hallucinations and artificial text in LLM-generated outputs often appear as distributional deviations between prompt and response hidden-state distributions.
Since prompts or retrieved contexts typically serve as reference samples and responses as query samples, with major differences in length, these asymmetries motivate the use of change test statistics that treat the two samples differently.
We consider an asymmetric two-sample test ASK-NN based on the directed k-nearest-neighbor (k-NN) graph.
Our statistic counts reference points whose nearest neighbor in the pooled sample is also a reference point.
Under the permutation null, it admits an exact finite-sample conditional mean and variance; we further establish asymptotic normality and consistency under fixed alternatives.
ASK-NN is computationally effective and easy to implement.
Empirically, it is competitive with kernel and graph-based baselines on synthetic benchmarks, artificial-text detection, and LLM hallucination detection from token-level hidden states.
\footnote{Code is available at: https://anonymous.4open.science/r/asymmetrical-multivariate-two-sample-test-21F0/}
\end{abstract}

\section{Introduction}
Large language models (LLMs) are increasingly deployed in various applications with extended usage in decision-support pipelines~\cite{liu2024deepseek}.
Due to limited memory and information constraints in the training data end date, they are often complemented with retrieval-augmented generation (RAG) systems~\cite{lewis2020retrieval,fan2024survey}, where the model is expected to produce responses \emph{faithful} to the evidence.
A core obstacle to reliability is \emph{hallucination}: the model may generate statements that are unsupported by, or contradict, the given context~\cite{zhang2025siren}.

Detection approaches emphasize that hallucinations in RAG can be exposed by measuring the \emph{mismatch} between the context and the generation using internal model signals~\cite{ricco2025hallucination}.
For example, probabilistic distances between prompt and response hidden-state distributions (e.g., MMD with dependence-aware resampling) provide a training-free score that increases under hallucination, and can be made comparable across variable-length sequences via p-values computed from an estimated null distribution~\cite{mmdhalludetection}.
Another line of work analyzes attention-map structure~\cite{chuang2024lookback} and reports that certain attention heads exhibit robust hallucination signatures across datasets, which can be summarized by a topological divergence on attention-induced graphs~\cite{bazarova2025toha}.

While these methods account for diverse patterns encountered by an LLM, they mostly ignore the internal structure for most tasks.
One typically has a \emph{reference} distribution (e.g., embeddings of retrieved evidence or a trusted prompt) and a \emph{query} distribution (embeddings of the generated response).
Crucially, the operational goal is often \emph{asymmetric}: we would like to tightly control false alarms when the response is well grounded in the reference, while remaining sensitive to deviations that, e.g., indicate hallucination.
Asymmetry also arises when the length of the prompt and the generated response can be significantly different, making most approaches inefficient in the presence of a large context.
This asymmetry is not fully captured by classical symmetric two-sample tests, which treat both samples~interchangeably.

In this paper, we formalize hallucination detection and the artificial text detection problem as an asymmetric two-sample testing problem and develop a nearest-neighbor (NN) graph test tailored to this regime.
Nearest-neighbor methods are attractive in high dimensions because they avoid density estimation and rely only on local geometry.
The classical nonparametric NN coincidence test of Henze counts how often a point and its NN belong to the same sample across both groups, and is symmetric by design with provable strong statistical properties~\cite{henze1988multivariate}.
We propose and analyze a one-sided variant that counts NN coincidences \emph{only within the reference sample}.
This yields a statistic that (i) aligns with the reference-vs-query structure of hallucination detection, (ii) admits an exact finite-sample conditional mean and variance under the permutation null (given the pooled locations), and (iii) retains asymptotic normality and consistency properties analogous to the symmetric case.

Our work adopts a \emph{reference-sided} counting rule, focusing only on within-reference NN coincidences.
This asymmetric construction better matches operational settings such as hallucination detection, where one sample plays the role of an anchor distribution, and yields a statistic with tractable conditional moments and analogous large-sample behavior.
Our theoretical analysis closely follows the framework of~\citet{henze1988multivariate}, while adapting the statistic, null moments, asymptotic calibration, and consistency argument to the asymmetric one-sided setting.
We further demonstrate how such asymmetric tests can be applied to real-world LLM generation problems using hidden-state representations.

Our contributions are as follows:
\begin{itemize}
    \item \textbf{An asymmetric two-sample statistic ASK-NN.} We introduce an asymmetric two-sample statistic ASK-NN based on within-reference coincidences in the k-NN graph, motivated by reference-query settings where one sample provides an anchor distribution.
    \item \textbf{Established theoretical properties for ASK-NN.} For ASK-NN statistic, we derive exact conditional mean and variance under the permutation null and obtain an asymptotically normal calibration. Our theoretical contributions continue with characterization of the limiting separation functional under fixed alternatives and show that it is minimized if and only if the two distributions coincide, implying consistency of the resulting test.
    \item \textbf{Competitive empirical performance of ASK-NN.} We evaluate the method on synthetic Gaussian benchmarks, artificial-text detection, and LLM hallucination detection, showing competitive performance against MMD, Sinkhorn, Hotelling’s \(T^2\), and symmetric k-NN~baselines.
\end{itemize}

\section{Related works}

\paragraph{Graph-based and nearest-neighbor tests.}
Graph-based tests form a class of nonparametric two-sample tests. They rely only on the geometric structure of the pooled sample, rather than on a parametric model for the underlying distributions. In particular, they construct a graph on pooled observations, for example using nearest-neighbor graphs or minimum spanning trees (MSTs), and measure how frequently edges connect points from the same sample.
Among the well-known examples are the Friedman–Rafsky test, based on the MST~\cite{friedman1979multivariate}, as well as nearest-neighbor tests, such as the Henze nearest-neighbor coincidence test~\cite{henze1988multivariate}. 
The latter counts the number of nearest-neighbor pairs within the samples symmetrically across both samples and establishes asymptotic normality under the null hypothesis and consistency under alternative hypotheses.

\paragraph{Two-sample testing and kernel methods.}
The two-sample problem has a long history, with modern nonparametric methods including kernel-based tests and graph-based tests.
Maximum Mean Discrepancy (MMD) is a widely used kernel two-sample statistic with strong theoretical guarantees and practical performance~\cite{gretton2012kernel}. 
For dependent data, wild-bootstrap variants provide valid calibration for degenerate kernel tests~\cite{chwialkowski2014wild}.
Connectivity between two samples is also explored in ~\cite{friedman1979multivariate,henze1999multivariate}, where the authors consider the number of edges in the minimum spanning tree connecting nodes from different samples.
Similar ideas appear in works~\cite{barannikov2021manifold}, but from a topological perspective, while with proved statistical properties~\cite{mironenko2026density}.
These methods are directly relevant to token-sequence embeddings in LLMs, where strong local dependence violates i.i.d.\ assumptions.

\paragraph{Hallucination detection via internal signals.}
A growing body of work detects hallucinations without external knowledge by extracting model-internal uncertainty or consistency cues, e.g., self-consistency over multiple generations~\citep{farquhar2024detecting}, entropy-based confidence~\citep{fadeeva2024fact}, and probes over hidden states~\citep{sky2024androids}.
In RAG, a particularly effective principle is to directly quantify the relationship between the provided context and the generated response~\citep{es2024ragas}.
One approach measures probabilistic distances between prompt and response hidden-state distributions and uses dependence-aware resampling (e.g., wild bootstrap) to compute calibrated p-values, yielding length-normalized hallucination scores~\citep{mmdhalludetection}.
In parallel, TOHA analyzes graphs induced by attention matrices and uses an asymmetric topological divergence between prompt and response subgraphs, reporting that a small subset of attention heads yields stable hallucination signatures across datasets and models~\citep{bazarova2025toha}.

\paragraph{Artificial-text detection.}
Machine-generated text (MGT) detection aims to distinguish human-written texts from outputs of large language models. This task has become increasingly challenging because modern LLMs generate fluent, high-quality text, making the distributional differences between human and generated writing subtle.

Existing methods include metric-based detectors based on language-model statistics such as likelihood, entropy, rank or perturbation-based scores \citep{gehrmann2019gltr, mitchell2023detectgpt}, model-based classifiers trained to separate human and generated examples
\citep{guo2023close}, and distributional approaches that explicitly view detection through the lens of statistical two-sample problem. In particular, recent work uses MMD-based criteria not only to measure discrepancies between human-written and machine-generated texts, but also to optimize detectors that are sensitive to such distributional differences \citep{zhangs2024MMDMP}.

We use this domain as a source of real-world data, where human-written and LLM-generated texts form two closely related distributions, providing a useful benchmark for evaluating the sensitivity of our two-sample statistic.

\section{Methods}

\subsection{Hypothesis testing problem}
Let $\vecX_1, \ldots, \vecX_{n_1}, \vecY_1, \ldots, \vecY_{n_2}$ be independent $\mathbb{R}^d$-valued random vectors with the dimension $d \ge 1$. Throughout the paper, the \(X\)-sample is treated as the reference sample and the \(Y\)-sample as the query or potentially shifted sample.
The distribution of $X_i$ has an unknown probability density function $f(x)$, and $Y_i$ has an unknown probability density function $g(x)$. 
We consider the two-sample testing problem $H_0: f \equiv g$ against $H_1: f \not\equiv g$.

\subsection{Nearest-neighbour based statistic}

Let 
\[
\vecZ_i = 
  \begin{cases}
    \vecX_i, & 1 \le i \le n_1, \\
    \vecY_{i - n_1}, & n_1 + 1 \le i \le n,
   \end{cases}
\]
where $n = n_1 + n_2$.
Then, for the pooled sample $\{Z_i\}$, we define $N_r(Z_i)$ to be the $r$-th nearest neighbor of $Z_i$ in the full sample, and $A(Z_i) = I(1 \le i \le n_1)$ is the indicator of the first group of points.

The statistic ASK-NN to be studied is:
\begin{equation}
    \label{eq:T_n_def}
    T_n = \sum_{i = 1}^{n} \sum_{r=1}^k A(Z_i) A(N_r(Z_i)). 
\end{equation}
Thus, \(T_n\) counts the number of reference points whose $k$ nearest neighbors in the pooled sample are also reference points, as it equals $\sum_{i = 1}^{n_1} \sum_{r=1}^k A(N_r(Z_i))$. 
Our goal is to examine the properties of $T_{n}$ and compare them with those of the version proposed by Henze $T^{\mathrm{full}}_n = \sum_{i = 1}^{n} \sum_{r=1}^k A(Z_i) A(N_r(Z_i)) + (1 - A(Z_i))(1 - A(N_r(Z_i)))$.
We will show that $T_n$ has properties that are largely similar to those of the original, while requiring a smaller amount of computation, as evident from Table~\ref{tab:comp_cost} for~$n_1 \ll n$.

\begin{table}[t]
\centering
\small
\begin{tabular}{lc}
\toprule
Method & Time \\
\midrule
MST-based test~(Henze et al., 1999) & $O(n^2 d)$ \\
MMD~\cite{gretton2012kernel} & $O(n^2 d)$ \\
MTopDiv~\cite{mironenko2026density} & $O(n^2 (d + \log n))$ \\
$T_n^{\mathrm{full}}$~\cite{henze1988multivariate} & $O(n (k + n d))$ \\
$T_n$ (ours) & $O(n_1 (k + n d))$ \\
\bottomrule
\end{tabular}
\caption{Computational costs for different nonparametric two-sample statistics.}
\label{tab:comp_cost}
\end{table}

\subsection{Distribution under $H_0$}

If $H_0$ holds, we can represent the random sequence $Z_1, \ldots, Z_n$ using the permutation distribution, similarly to~\citet{henze1988multivariate}. 
$Z_1, \ldots, Z_n$ are i.i.d random $d$-dimensional vectors.
Additional random variables $U_{n1}, \ldots, U_{nn}$ have distribution:
\begin{multline*}
    \mathbb{P}(U_{nj} = u_j; 1 \le j \le n) = \\
    = \begin{cases}
        \binom{n}{n_1}^{-1}, & \text{if} \sum_{i=1}^{n} I(u_i = 1) = n_1, \\
        0, & \text{otherwise}.
    \end{cases},
\end{multline*}
where $u_j \in \{1, 2\}$ are all possible values of $U_{nj}$. 
These variables represent the sample type of $Z_j$: $X$ or $Y$-type correspondingly. For $1 \le i \neq j \le n, 1 \le r \le k$ we introduce events
\begin{equation*}
    A_{ij}^{(r)} = [Z_j = N_r(Z_i)],
\end{equation*}
\begin{equation*}
    B_j = [U_{nj} = 1].
\end{equation*}

Here we study only one sample type ($U_{nj} = 1$), therefore, the test becomes asymmetric.

Under $H_0$, $T_n$ defined in \eqref{eq:T_n_def} has the same distribution as
\begin{equation*}
    \tilde{T}_n = \sum_{i,j = 1}^{n} \sum_{r=1}^k I\left(B_i\right) I\left(B_j\right) I\left(A_{ij}^{(r)}\right).
\end{equation*}
Given $Z_i = z_i$ we denote:
\begin{equation*}
    a_{ij}^{(r)} = a_{ij}^{(r)}(z_1, \ldots, z_n) = I(A_{ij}^{(r)} | Z_1 = z_1, \ldots, Z_n = z_n).
\end{equation*}
\begin{equation*}
    a_{ij}^+ = \sum_{r=1}^k a_{ij}^{(r)} \quad (a_{ij}^+ \in \{0, 1\})
\end{equation*}
All $a_{ij}^+$ together form an adjacency matrix of the k-nearest-neighbour graph on points $z_1, \ldots, z_n$.
\begin{equation*}
    L_n = \sum_{i, j = 1}^{n} a_{ij}^+ I(B_i) I(B_j).
\end{equation*}

For convenience, let $d_j^{(k)} = \sum_{i=1}^{n}a_{ij}^+$, $c_n^{(k)} = \frac{1}{nk}\sum_{j = 1}^{n}\left(d_j - k\right)^2$, $v_n^{(k)} = \frac{1}{nk}\sum_{i, j = 1}^{n}a_{ij}^+a_{ji}^+$ (same notation is used in~\cite{henze1988multivariate}). Now we can formulate the one-sided analog of Proposition 2.1 from~\cite{henze1988multivariate}:
\begin{prop}\label{prop:L_moments} For the random variable $L_n$ it holds that
    \begin{equation}\label{eq:E_L}
        \mathbb{E}[L_n] = k\frac{n_1(n_1 - 1)}{n - 1},
    \end{equation}
    \begin{multline}\label{eq:Var_L}
        \mathbb{V}[L_n] = k\frac{n_1 (n_1 - 1) n_2}{(n-1) (n-2) (n-3)} \times \\
        \times \left((n_1 - 2) c_n^{(k)} + (n_2 - 1) \left(1 + v_n^{(k)} - \frac{2k}{n - 1}\right)\right).
    \end{multline}    
\end{prop}

\begin{proof}
    For the proof, we turn to the work~\cite{Bloemena1964SamplingGraph}. 
    Letting $m_{ij} = a_{ij}^+ + a_{ji}^+, Q = \sum_{i,j = 1}^{n} m_{ij} I(B_i) I(B_j)$, we have $L_n = \frac{1}{2} Q$. The variable $Q$ ($x_W$ or $x_B$) was studied by Bloemena (see def. (1.1.5)) and the assertion follows from (3.3.1) and (3.3.8), observing that, in their notation, $m_{i+} = k + d_i^{(k)}$, $m_{++} = 2nk$, $\sum_{i=1}^{n} (m_{i+} - (1/n) m_{++})^2 = nk c_n$ and $\sum_{i, j=1}^{n} m_{ij}^2 = 2nk (1 + v_n)$.
    So, we obtain the desired results by straightforwardly calculating the mean and variance in question.
\end{proof}

Now we derive a new asymptotic distribution under $H_0$.

\begin{prop}
    If $n \rightarrow \infty$ with $n_1/n \rightarrow \tau, 0 < \tau < 1$, then
    \begin{equation*}
        \mathbb{V} \left(n^{-1/2} \tilde{T_n} | Z_1, \ldots, Z_n\right) \rightarrow_{P_f} \sigma^2(\tau, d, |.|),
    \end{equation*}
    where
    \begin{equation*}
        \sigma^2(\tau, d, |.|) = k\tau^2(1-\tau)\left(\tau c_\infty^{(k)} + (1 - \tau)(1 + v_\infty^{(k)})\right),
    \end{equation*}
    $v_\infty^{(k)}$ and $c_\infty^{(k)}$ are defined in Propositions 3.1, 3.2 in~\citet{henze1988multivariate}.
\end{prop}

\begin{proof}
    The result follows immediately taking to the limit \eqref{eq:Var_L}, using Propositions 3.1, 3.2 in~\citet{henze1988multivariate} about $c_\infty^{(k)}$ and $v_\infty^{(k)}$.    
\end{proof}

Then, using the fact above (or using Theorem 4.1.1 from~\citet{Bloemena1964SamplingGraph} with notations given in the proof of Prop. \ref{prop:L_moments}), we obtain the next result:

\begin{theorem}
    If $n \rightarrow \infty$ with $n_1/n \rightarrow \tau, 0 < \tau < 1$, then
    \begin{equation*}
        n^{-1/2}\left(T_n - k\frac{n_1 (n_1 - 1)} {n - 1}\right) \rightarrow_{\mathcal{D}_f} 	\mathcal{N}\left(0, \sigma^2(\tau, d, |.|)\right),
    \end{equation*}
    \begin{equation*}
        \lim \mathbb{V} \left(n^{-1/2} T_n\right) = \sigma^2(\tau, d, |.|).
    \end{equation*}
\end{theorem}

The result above delivers the distribution of $T_n$ under $H_0$.
It allows us to derive the p-values for $T_n$ given its asymptotical normality and analytical expressions for its mean and variance.

\subsection{Consistency}

Now we turn our attention to a one-sided analog of Theorem 4.1 from~\citet{henze1988multivariate} that describes what happens if $f \ne g$.

\begin{theorem}\label{Metric theorem}
    If $n \rightarrow \infty$ with $n_1/n \rightarrow \tau, 0 < \tau < 1$, then
    \begin{equation*}
        \frac{1}{nk}T_n \rightarrow_{P_{f,g}} \tilde{D}(f, g, \tau),
    \end{equation*}
    where
    \begin{equation*}
        \tilde{D}(f, g, \tau) = \int \frac{\tau^2f^2(x)}{\tau f(x) + (1 - \tau)g(x)}dx.
    \end{equation*}
\end{theorem}

\begin{proof}
    Using the notation $I_j(r)$ (definition (1.3) from~\citet{henze1988multivariate}), the result follows from Lemma 4.2 from~\citet{henze1988multivariate} and the fact that
    \begin{multline*}
        \lim \mathbb{E}\left[I_1(1)I_2(1) | X_1 = x_1, X_2 = x_2\right] = \\
        = \prod_{j=1}^2\frac{\tau f(x_j)}{\tau f(x_j) + (1-\tau) g(x_j)}.
    \end{multline*}
    
    which is given in~\citet{henze1988multivariate} to prove Theorem 4.1. 
\end{proof}

Now we prove a special case of inequality for this asymmetric separation measure $\tilde{D}$: 

\begin{prop}\label{Coinside prop}
    Let $f, g$ be probability density functions on $\mathbb{R}^d$, and let $0 < \tau < 1$. Then
    \begin{equation*}
        \tilde{D}(f, g, \tau) \ge \tau^2.
    \end{equation*}
    Equality holds, if and only if, the probability measures corresponding to $f$ and $g$ coincide.
\end{prop}

\begin{proof}
    Let us note that
    \begin{multline*}
        \tilde{D}(f, g, \tau) - \tilde{D}(g, f, 1 - \tau) = \\ 
        = \int \frac{\tau^2 f^2(x) - (1-\tau)^2g^2(x)}{\tau f(x) + (1-\tau)g(x)}dx = \\
        = \int \left(\tau f(x) - (1 - \tau)g(x)\right)dx = \\
        = \tau - (1-\tau) = 2\tau - 1
    \end{multline*}
    and
    \begin{equation*}
        \tilde{D}(f, g, \tau) + \tilde{D}(g, f, 1 - \tau) = D(f, g, \tau),
    \end{equation*}
    where $D(f, g, \tau)$ is given in statement of Theorem 4.1 from~\citet{henze1988multivariate}. Therefore:
    \begin{equation*}
        \tilde{D}(f, g, \tau) = \frac{D(f, g, \tau) + 2\tau - 1}{2}.
    \end{equation*}
    Following Proposition 4.3 from~\citet{henze1988multivariate} we have:
    \begin{equation*}
        D(f, g, \tau) \ge \tau^2 + (1 - \tau)^2 \Leftrightarrow \tilde{D}(f, g, \tau) \ge \tau^2
    \end{equation*}
    and equality holds if, and only if, $f$ and $g$ coincide.
\end{proof}

After that the statistical test is defined similarly (see Theorem 4.4 from~\citet{henze1988multivariate}), aside from different formulas for $m(n_1, n_2)$ and $U_{n_1, n_2}$:
\begin{equation*}
    m(n_1, n_2) = \frac{n_1(n_1 - 1)}{n - 1},
\end{equation*}
\begin{multline*}
    U_{n_1, n_2} = k\frac{n_1(n_1-1)n_2}{(n-1)(n-2)(n-3)} \times \\
    \times \left((n_1 - 2)C_n + (n_2-1)\left(1+V_n-\frac{2k}{n-1}\right)\right) \le \\
    \le k\frac{n_1(n_1-1)n_2}{(n-1)(n-2)(n-3)} \times \\
    \times \left((n_1 - 2)(\mathcal{C} + 1)^2 + 2(n_2 - 1)\right) = O_{\mathcal{P}_{f,g}}(n),
\end{multline*}

where $C_n$, $V_n$, $\mathcal{C}$ are defined in (3.1), (3.2), (2.5) from~\citet{henze1988multivariate}, respectively. Using this we can derive (4.8) from~\citet{henze1988multivariate}. The last inequality in the proof of Theorem 4.4 from~\citet{henze1988multivariate} now follows from Theorem \ref{Metric theorem}, (4.8) from~\citet{henze1988multivariate}, Proposition \ref{Coinside prop} and the fact that, as $n \rightarrow \infty, n_1/n \rightarrow \tau$,
\begin{equation*}
    \lim\left(n^{-1}m(n_1, n_2)\right) = \tau^2.
\end{equation*}

The proposed test is defined similarly as in Remark 5.1 from~\citet{henze1988multivariate}:
\begin{remark} \label{test_remark}
    For moderate or large sizes of $n_1, n_2$ we may reject $H_0$ at (approximate) level $\alpha$ if
    \begin{equation*}
        T_n(\mathbf{z}) \ge c_n^*(\mathbf{z}; \alpha),
    \end{equation*}
    where $\mathbf{z}$ and $c_n^*(\mathbf{z}; \alpha) = c_{n,1}^*(\mathbf{z}; \alpha)$ are defined similarly with $m(n_1, n_2)$ and $u_{n_1, n_2}$ now equal to right-hand sides of \eqref{eq:E_L} and \eqref{eq:Var_L}, respectively.
\end{remark}

\section{Experiments}

\subsection{Baselines and evaluation protocol}
\paragraph{Baselines.}
We compare the proposed test \textbf{ASK-NN} against representative baselines from several families of two-sample tests: Hotelling’s \(T^2\) test (\textit{t-test}) as a classical parametric mean-based test; the original Henze’s 1-NN graph symmetry test (\textit{Symmetric})~\citep{henze1988multivariate}; the unbiased Maximum Mean Discrepancy with RBF-kernel, label shuffling and median bandwidth heuristic (\(\mathrm{MMD}_{u}B\)) as a kernel-based two-sample test~\citep{mmdagg}; and Sinkhorn Divergency (\textit{SD}) as the entropy-regularized optimal transport version~\citep{sinkhorn}. For the SD baseline, we use a small entropic regularization coefficient to stabilize the computation and compute the statistic efficiently via Sinkhorn iterations, while preserving a close approximation to the unregularized OT distance.
For the \(\mathrm{MMD}_{u}B\), \textit{SD} and \textit{t-test} baselines, p-values are obtained by the same label-shuffling procedure~\citep{labelshuffling}.

\paragraph{Monte Carlo evaluation.}
For each configuration, we estimate power as the empirical rejection rate over \(R\) independent Monte Carlo repetitions at the significance level \(\alpha = 0.05\). 
Error bars indicate normal-approximation \(95\%\) Monte Carlo confidence intervals,
\[
    \hat p \pm 1.96 \sqrt{\frac{\hat p(1-\hat p)}{R}},
\]
where \(\hat p\) is the empirical rejection rate and \(R\) is the number of repetitions.
\subsection{Computational complexity}
\paragraph{Implementation details.} To study practical computational scaling, we benchmark all methods on synthetic Gaussian inputs. For each configuration, we generate two samples of size \(N\) in dimension \(d\) and measure wall-clock runtime of the statistic computational only, excluding data generation. We consider dimensions \(d \in \{128, 1024\}\) and vary the total sample size from \(512\) to \(32768\) on a powers-of-two grid. Experiments are performed on an NVIDIA A100 GPU after 10 warm-up iterations. Runtime is reported as the median over repeated 20 executions.

\paragraph{Results.} Figure~\ref{fig:computational_complexity} presents the observed runtime scaling of the proposed ASK-NN statistic and competing baseline statistics. \textbf{ASK-NN} exhibits substantially slower runtime growth than the kernel-based \(MMD_{u}\) statistic and the \textit{SD} baseline. Moreover, ASK-NN remains applicable at larger sample sizes than these baselines, which become memory-limited earlier due to their higher memory requirements.

Compared with the symmetric nearest-neighbor statistic, ASK-NN shows similar scaling behavior, with slightly lower runtime than this baseline.
Hotelling's \(T^2\) test exhibits nearly constant runtime in the explored regime, likely because its computation is dominated by fixed GPU overheads rather than arithmetic cost.

Overall, these experiments suggest that nearest-neighbor-based two-sample statistics offer better practical scalability than kernel-based alternatives, while ASK-NN preserves this advantage and remains computationally competitive with the symmetric nearest-neighbor baseline.

\begin{figure*}[t]
    \centering

    \includegraphics[width=0.80\textwidth]{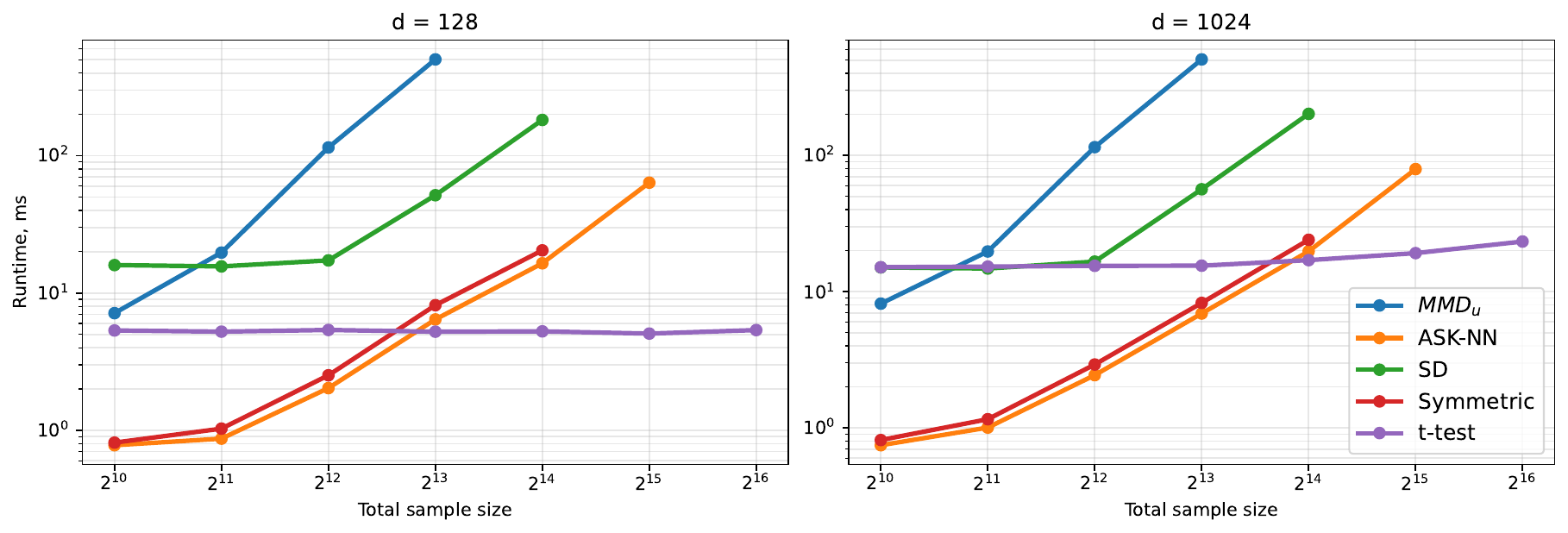}

    \caption{Observed wall-clock runtime scaling of the proposed ASK-NN statistic and baseline two-sample statistics on an NVIDIA A100 GPU for dimensions \(d = 128\) (\textbf{Left}) and \(d = 1024\) (\textbf{Right}). Runtime is measured only for statistic computation, excluding data generation. Both axes use logarithmic scales. Missing points correspond to configurations that exceeded the available GPU memory.}
    \label{fig:computational_complexity}
\end{figure*}
\subsection{Synthetic data}
    
\paragraph{Gaussian benchmarks.}
Similar to~\citet{gretton2012kernel}, we evaluate all tests on two high-dimensional Gaussian two-sample problems. 
In the mean-shift setting, we sample
\(X \sim \mathcal{N}(0, I_d)\) and
\(Y \sim \mathcal{N}(\mu_\delta, I_d)\), where
\(\mu_\delta = \delta \mathbf{1}_d / \sqrt{d}\). 
This choice keeps the Euclidean norm of the mean difference fixed, \(\|\mu_\delta\|_2 = \delta\), independently of the dimension \(d\). 
In the scale-shift setting, we sample
\(X \sim \mathcal{N}(0, I_d)\) and
\(Y \sim \mathcal{N}(0, \sigma^2 I_d)\), where \(\sigma\) controls the variance difference.
We evaluate dimensions ranging from \(2\) to \(2500\).
For the mean-shift experiments, we use 20 logarithmically spaced values of
\(\delta\) in \([0.05, 50]\). 
For the scale-shift experiments, we use 20 logarithmically spaced values of
\(\sigma\) in \([1, 10]\).
For each configuration, we estimate power as the fraction of repetitions in which the null hypothesis is rejected at level \(\alpha = 0.05\). In all experiments, we draw \(n_1=n_2=250\) samples from each distribution.

\paragraph{Results.}
Figure~\ref{fig:gaussian_experiments} demonstrates empirical power under the mean-shift and variance-shift alternatives.\\
In the mean-shift setting, MMD and Hotelling's \(T^2\) test achieve the highest rejection rates. This is consistent with the structure of the alternative: Hotelling's \(T^2\) test directly targets differences in means, while MMD with a characteristic RBF-kernel is sensitive to the induced difference between the two Gaussian distributions. The proposed ASK-NN does not outperform the considered baselines in this setting, but attains nontrivial empirical power, confirming that the statistic remains sensitive to mean-shift alternatives.

In the variance-shift setting, the ranking changes substantially: the proposed asymmetric test achieves the strongest performance among all considered methods and remains highly powerful across the full range of dimensions. 
MMD is the closest competitor and also performs robustly, whereas the Sinkhorn divergence baseline is substantially weaker in higher dimensions. 
Henze's original symmetric 1-NN graph test loses power almost completely, highlighting the advantage of the proposed asymmetric construction for detecting variance changes in high dimensions.


\begin{figure}[t]
    \centering

    \includegraphics[width=0.80\columnwidth]{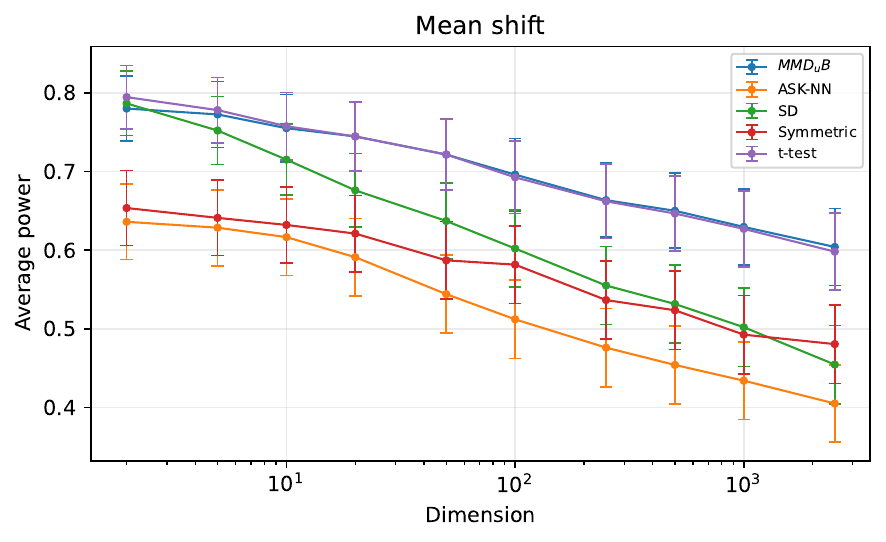}

    \vspace{0.5em}

    \includegraphics[width=0.80\columnwidth]{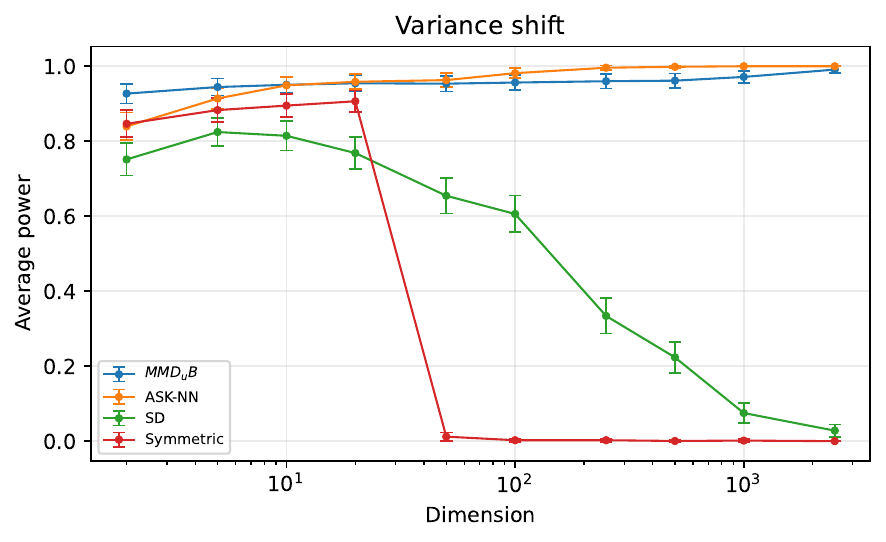}

    \caption{
    Empirical power on synthetic Gaussian two-sample benchmarks at a given significance level \(\alpha = 0.05\).
    \textbf{Top}: Gaussians having the same variance and different means.
    \textbf{Bottom}: Gaussians having the same mean and different variances.}
    \label{fig:gaussian_experiments}
\end{figure}
\subsection{Real problem of Artificial text detection}

\paragraph{Experimental setup.}
We further evaluate the proposed test on artificial-text detection datasets.
We use the RAID dataset~\citep{raid} and consider the subset corresponding to GPT-4 generations in the abstracts domain. 
For each sample size, we draw two independent samples of texts: one from the human-written subset and one from the GPT-4-generated subset. 
Texts are embedded using the multilingual-e5-large sentence encoder~\citep{multilinguale5}, and the resulting embedding samples are compared using the considered two-sample tests. 
This setup evaluates whether the tests can detect the distributional difference between human-written and machine-generated texts in a semantic embedding space.

Since the original asymptotic calibration of the proposed statistic is not sufficiently well-calibrated in this setting, we also evaluate a label-shuffling variant of our test, denoted \textit{Asymmetric (LS)}, where \(p\)-values are obtained by permuting labels in the pooled sample.

\paragraph{Results.}
Figure~\ref{fig:artificial_detection} reports Type II and Type I errors across sample sizes. 
The results show that human-written and GPT-4-generated texts are well separated in the multilingual-e5-large embedding space, with most criteria achieving low Type II error at moderate sample sizes.

The proposed asymmetric statistic is also sensitive to this distributional difference and reaches low Type II error as the sample size increases.
However, the original version requires careful calibration in this setting, as its Type I error may deviate from the nominal significance level on real text embeddings. To address this issue, we additionally evaluate a label-shuffling calibration of our approach (\textit{ASK-NN (LS)}). 
As the results show, this variant keeps Type I error close to the nominal level across the considered sample sizes, while retaining strong power. 

The comparison between the original and label-shuffling variants indicates that the main challenge in this setting is calibration rather than lack of signal: our proposed statistic captures the human-vs-generated difference, while resampling provides reliable Type I error control.

\begin{figure}[t]
    \centering
    \includegraphics[width=0.8\columnwidth]{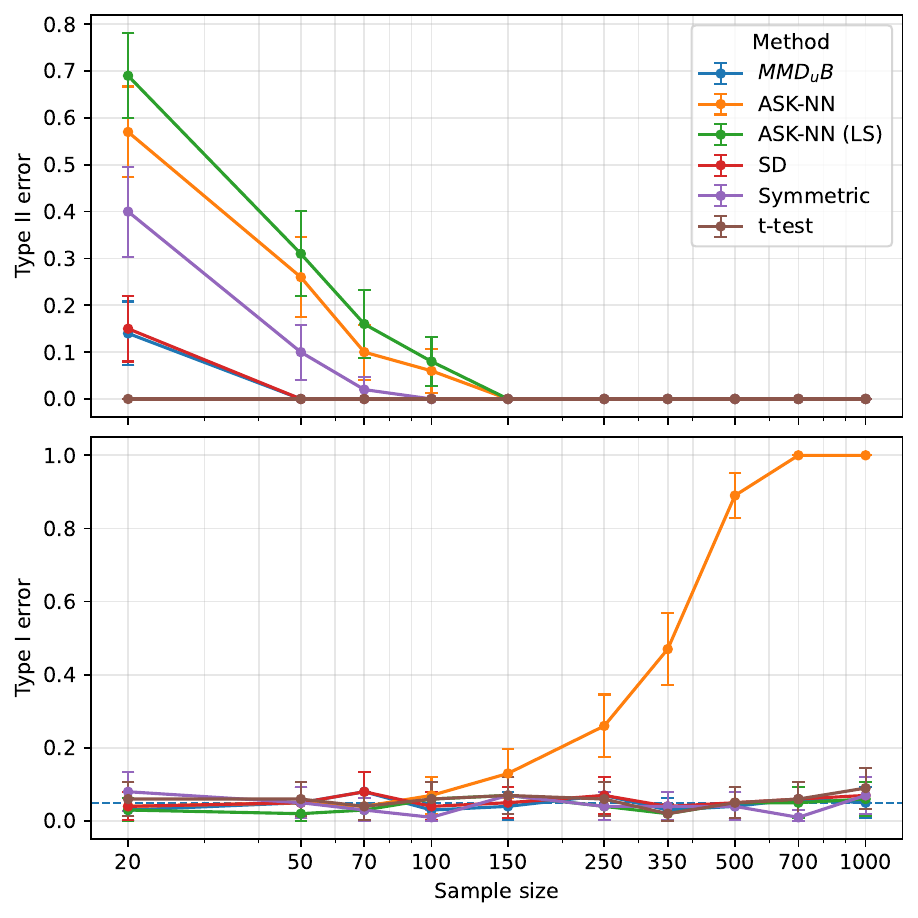}
    \caption{Type II error and Type I error on the RAID artificial-text detection dataset as a function of sample size. The label-shuffling version of the proposed asymmetric statistic maintains Type I error close to the nominal level while retaining strong power.}
    \label{fig:artificial_detection}
\end{figure}

\subsection{Real problem of LLM Hallucination Detection}

\paragraph{Experimental setup.}
We evaluate the proposed test in an applied hallucination detection task following the experimental protocol of~\citet{mmdhalludetection}. 
The task is motivated by RAG, in which hallucinated responses are expected to exhibit weaker or structurally different relationships with the provided context. 
Motivated by this, each prompt-response pair is represented by the token-level hidden states of an LLM, and hallucination detection is formulated as a two-sample problem comparing the hidden-state distributions of prompt and response tokens. 
For a prompt-response pair \(S = (P, R)\), token-level LLM hidden states define two empirical samples,
\(X_P=\{h_i^P\}_{i=1}^{m}\) and \(Y_R = \{h_j^R\}_{j=1}^{n}\), corresponding to prompt and response tokens. 
We compute the proposed asymmetric graph-based statistic for the prompt and response hidden-state samples (\(X_P, Y_R\)) and use it as a hallucination score.

We evaluate our approach and the baselines on a subset of the benchmark considered by~\citet{mmdhalludetection}: long-form QA dataset MS MARCO and summarization CNN/DM + Recent News from RAGTruth~\citep{ragtruth}, and conversational QA benchmark CoQA~\citep{coqa}, using two open-source LLMs: Mistral-7B-Instruct-v0.1 and LLaMA-2-7B-chat.

\paragraph{Results.}
Table~\ref{tab:roc-auc-hallucination} shows ROC-AUC scores for hallucination detection on three datasets and two LLMs. 
The proposed asymmetric statistic achieves the best or second-best result in five out of six settings. 
For Mistral-7B, it matches the best performance on MS MARCO and CNN/DM + Recent News, achieving ROC-AUC scores of \(0.72\) and \(0.62\), respectively, and remains close to the strongest method on CoQA. 
For LLaMA-2-7B, the proposed method obtains the best results on MS MARCO and CoQA, and is second-best on CNN/DM + Recent News.

Compared to the MMD-based criterion, our approach improves performance on most datasets and model backbones, suggesting that the proposed statistic captures prompt-response mismatch in hidden-state space more effectively in this application. 
The gains are particularly clear for LLaMA-2-7B, where the proposed method outperforms MMD on all three datasets. 
Overall, these results indicate that our approach transfers well to a practical hallucination detection scenario, beyond the controlled synthetic alternatives considered above.
\begin{table}[t]
\centering
\caption{ROC AUC ($\uparrow$) of different statistical criteria in the task of hallucination detection for two LLMs. The best results for each model are highlighted in \textbf{bold}, and the second best are \underline{underlined}.}
\label{tab:roc-auc-hallucination}

\begingroup
\resizebox{0.90\columnwidth}{!}{%
\begin{tabular}{l c c c}
\toprule
Method & MS MARCO & \makecell{CNN/DM +\\Recent News} & CoQA\\
\midrule

\multicolumn{4}{c}{Mistral-7B}\\
\midrule

$MMD_{u}B$ &
\metricsecond{0.69}{0.03} &
\metricsecond{0.60}{0.05} &
\metricbest{0.73}{0.03}\\

SD &
\metricplain{0.58}{0.03} &
\metricplain{0.49}{0.05} &
\metricplain{0.62}{0.07}\\

t-test &
\metricplain{0.61}{0.02} &
\metricplain{0.59}{0.07} &
\metricplain{0.57}{0.07}\\

Symmetric &
\metricbest{0.72}{0.03} &
\metricbest{0.62}{0.04} &
\metricsecond{0.70}{0.06}\\

\textbf{ASK-NN (ours)} &
\metricbest{0.72}{0.03} &
\metricbest{0.62}{0.03} &
\metricplain{0.69}{0.07}\\
\midrule

\multicolumn{4}{c}{LLaMA-2-7B}\\
\midrule

$MMD_{u}B$ &
\metricplain{0.64}{0.03} &
\metricplain{0.53}{0.03} &
\metricplain{0.68}{0.04}\\

SD &
\metricplain{0.51}{0.02} &
\metricplain{0.47}{0.03} &
\metricsecond{0.75}{0.03}\\

t-test &
\metricplain{0.60}{0.05} &
\metricplain{0.51}{0.02} &
\metricplain{0.72}{0.06}\\

Symmetric &
\metricsecond{0.65}{0.03} &
\metricbest{0.55}{0.02} &
\metricsecond{0.75}{0.04}\\

\textbf{ASK-NN (ours)} &
\metricbest{0.66}{0.02} &
\metricsecond{0.54}{0.02} &
\metricbest{0.78}{0.03}\\

\bottomrule
\end{tabular}
}
\endgroup

\end{table}

\section{Conclusion}
We introduced an asymmetric multivariate two-sample test based on directed 1-nearest-neighbor coincidences. Unlike symmetric graph-based tests, the proposed statistic treats one sample as a reference distribution and measures how its local neighborhood structure changes when a query sample is introduced. We derived exact finite-sample conditional moments under the permutation null, established asymptotic normality, and proved consistency under fixed alternatives.

Experiments on synthetic benchmarks, artificial-text detection, and LLM hallucination detection show that the proposed statistic is not only theoretically tractable and meaningful on controlled Gaussian alternatives, but also remains competitive on real embedding-based tasks.

Several directions remain open. First, the asymptotic calibration can be inaccurate for real hidden-state embeddings, where dependence, anisotropy, and high dimensionality are substantial; permutation or block-resampling calibration may provide more robust alternatives. Second, extending the theory to dependent token-level samples and high-dimensional regimes would make the method better aligned with LLM applications and improve hallucination detection in practical RAG systems.
\bibliography{refs}

\end{document}